\pdfoutput=1
\documentclass[10pt]{article}
\usepackage[utf8]{inputenc}
\usepackage{times}
\usepackage[margin=1in]{geometry}
\usepackage{amsmath,amssymb,amsthm}
\usepackage{graphicx}
\usepackage{booktabs}
\usepackage{microtype}
\usepackage{xcolor}
\usepackage[round]{natbib}
\usepackage[hidelinks]{hyperref}
\usepackage{authblk}

\newtheorem{proposition}{Proposition}
\newcommand{\pio}{\pi}
\newcommand{\TPR}{\mathrm{TPR}}
\newcommand{\FPR}{\mathrm{FPR}}
\title{\bf Prevalence Determines Precision:\\ Silent Contamination in Detector-Defined Datasets}
\author{
  Jia Huang\textsuperscript{1}, Yankai Wan\textsuperscript{2},Yangjun Ou\textsuperscript{3}
}
\affil{
  \textsuperscript{1}Guanghua School of Management, Peking University\\
  \textsuperscript{2}College of Artificial Intelligence, Jilin University\\
  \textsuperscript{3}School of Mathematical Sciences, Peking University
}
\date{}

\begin{document}
\maketitle

\begin{abstract}
A large share of machine learning datasets are not observed but \emph{constructed}: a detector, heuristic, or model is run over a pool of candidates, and whatever it accepts becomes the dataset. Weak and distant supervision, pseudo-labelling, event extraction, and most anomaly-detection benchmarks all have this form. The precision of the resulting dataset is not a property of the detector. It is governed by the prevalence $\pio$ of true positives in the pool the detector is deployed on, through elementary Bayes. This is textbook, and it is nonetheless almost never measured end-to-end, for a structural reason: from inside a detector-defined dataset the false positives are unidentifiable, so the quantity that matters cannot be estimated by the people who need it. We report a case in which it \emph{can} be measured. For the same instrument and period we hold both a detector-defined event dataset and an independent official index that reveals, for every detected item, whether it is real. One detector, three candidate pools, three datasets: phantom rates of 81.7\%, 9.0\%, and 0.0\%. A practitioner transferring precision from the two high-$\pio$ pools to the low-$\pio$ pool predicts 0.955 against a measured 0.183, an error of $+422\%$; the Bayes expression predicts within $3.3\%$ across all three pools. Beyond confirming the mechanism we report three findings that we believe are new. First, the detected dataset's response curve is an \emph{exact} convex combination of a true-event and a phantom component (identity residual $1.1\times10^{-16}$), with phantoms outnumbering true events 473 to 308 --- contamination is not noise added to signal but a second signal with its own shape, inherited from the detector's acceptance rule. Second, the \emph{direction} of contamination is a property of the estimator, not of the data: on identical windows, one statistic shows contamination diluting the effect and another shows it inflating the effect, because the second statistic's denominator is itself contaminated. Third, a normalisation in common use turns the estimator into a mean of ratios whose expectation does not exist; on the same 335 events it returns 0.40 where the well-defined estimator returns 0.10.
\end{abstract}

\section{Introduction}
\label{sec:intro}

Consider how a dataset gets built when nobody is willing to label $10^6$ items by hand. A rule is written --- a regular expression, a knowledge-base join, a threshold on a model's confidence, a statistical scan --- and applied to a pool of candidates. What the rule accepts becomes the dataset. Downstream, the accepted items are treated as observations. This is the shape of distant supervision \citep{mintz2009}, of weak supervision \citep{ratner2017}, of pseudo-labelling, and of the great majority of event-detection and anomaly benchmarks.

The precision of such a dataset obeys
\begin{equation}
\text{precision} \;=\; \frac{\pio\,\TPR}{\pio\,\TPR + (1-\pio)\,\FPR},
\label{eq:bayes}
\end{equation}
where $\pio$ is the fraction of true positives in the pool the detector was run on. Precision is therefore a property of the \emph{deployment pool}. What transfers across pools is the pair $(\TPR,\FPR)$; precision does not. A detector validated where $\pio\approx1$ --- and at $\pio\approx1$ even a badly broken detector looks flawless --- and then redeployed where $\pio\ll1$ imports false positives at a rate its validation never measured.

None of this is new. \citet{barhillel1980} named the reasoning error; \citet{hand2009} and \citet{saito2015} state the consequence for classifier evaluation directly. What is missing is measurement. The obstacle is structural rather than sociological: inside a detector-defined dataset the false positives are exactly the items one cannot identify. If they could be identified, the detector would not have accepted them. Practitioners therefore validate on whatever labelled subset they have --- which is, almost by construction, a high-$\pio$ subset --- and deploy elsewhere.

\paragraph{This paper.} We report a setting where the phantoms \emph{are} identifiable. For E-mini S\&P~500 futures over 2010--2026 we hold two independent event indices over the same price series: a detector-defined set produced by a magnitude-threshold scan, and an official index (CPI, non-farm payrolls, and FOMC statement releases) published by the institutions that generate the events. Every detected item can be classified. Moreover the detector was run on three structurally different candidate pools --- the Federal Reserve's own meeting calendar, the set of first Fridays, and business days 6--16 of each month --- giving three deployments of one detector at $\pio\in\{1.00,\,0.82,\,0.10\}$ within a single dataset.

We use this to do three things a synthetic study cannot. We measure \eqref{eq:bayes} against reality rather than assuming it. We decompose a real contaminated estimate into its true and phantom components exactly, and characterise what the phantom component looks like. And we show that the contamination's effect on a downstream scientific conclusion is not even sign-stable across reasonable choices of estimator.

\paragraph{Contributions.}
\begin{enumerate}\itemsep1pt
\item \textbf{An end-to-end measurement of prevalence-driven contamination} with recoverable ground truth: one detector, three pools, phantom rates 81.7\% / 9.0\% / 0.0\%; naive cross-pool precision transfer errs by $+422\%$ while \eqref{eq:bayes} holds to $3.3\%$ (\S\ref{sec:pools}--\S\ref{sec:transfer}).
\item \textbf{An exact mixture decomposition} showing that a contaminated estimate is a convex combination of a true-event and a phantom estimate, verified to machine precision on real data, with the phantom component reproducible from the detector's acceptance rule applied to pure noise (\S\ref{sec:mixture}).
\item \textbf{Estimator-dependence of the contamination direction.} On identical windows, contamination dilutes under a denominator-free statistic and inflates under a terminally normalised one. We trace this to contamination of the estimator's own scale (\S\ref{sec:direction}).
\item \textbf{A non-integrable estimator in common use.} Per-item normalisation by a signed reference return yields a mean of ratios with no finite expectation. Empirically it returns 0.40 where the equal-weight estimator returns 0.10 on the same events (\S\ref{sec:ratio}).
\item \textbf{A seven-rule protocol} (\S\ref{sec:protocol}) and a six-accident appendix (App.~\ref{app:accidents}) documenting the mechanism operating on this paper's own pipeline.
\end{enumerate}

\section{Setup}
\label{sec:setup}

\paragraph{Observed and detected indices.} Let $p_t$ be a series and let $T^{\mathrm{obs}}$ be an \emph{observed} index published by an external authority. Let $A$ be a \emph{detector}: any deterministic rule mapping a candidate window to accept/reject. The \emph{candidate pool} $C$ is the set of windows at which $A$ is evaluated. The \emph{detected index} is $T^{\mathrm{det}}=A(C)$. Users of $T^{\mathrm{det}}$ typically see only $|T^{\mathrm{det}}|$.

\paragraph{Five quantities.} We argue five must be reported alongside any detector-defined dataset:
\begin{align*}
\pio &= \tfrac{|T^{\mathrm{obs}}\cap C|}{|C|} && \text{prevalence in the deployment pool}\\
\TPR &= \tfrac{|T^{\mathrm{obs}}\cap T^{\mathrm{det}}|}{|T^{\mathrm{obs}}\cap C|}, \qquad
\FPR = \tfrac{|T^{\mathrm{det}}\setminus T^{\mathrm{obs}}|}{|C\setminus T^{\mathrm{obs}}|} && \text{detector operating point}\\
c &= \tfrac{n_P}{n_E+n_P} && \text{phantom mass fraction}\\
\rho &= w_P/w_E && \text{phantom-to-true intensity ratio}
\end{align*}
where $n_E=|T^{\mathrm{obs}}\cap T^{\mathrm{det}}|$, $n_P=|T^{\mathrm{det}}\setminus T^{\mathrm{obs}}|$, and $w_E,w_P$ are the mean magnitudes of the two subsets under whatever scale the downstream estimator uses. Precision equals $1-c$, and by Bayes it equals \eqref{eq:bayes}.

The last two are the ones normally omitted even by careful authors. $c$ says how much of the dataset is fabricated; $\rho$ says how loud the fabricated part is relative to the real part. \S\ref{sec:mixture} shows they jointly determine the damage, and that $c$ alone does not.

\paragraph{Response curves.} For an index $T$ with anchor sign $s_i=\operatorname{sign}\,r_i(\tau_{\mathrm{ref}})$, define
\begin{equation}
S(\tau)=\frac{1}{n}\sum_i s_i\, r_i(\tau),
\qquad
D_{\mathrm{ref}}(\tau)=\frac{S(\tau)}{\frac1n\sum_i |r_i(\tau_{\mathrm{ref}})|},
\qquad
\widetilde D(\tau)=\frac{1}{n}\sum_i \frac{r_i(\tau)}{r_i(\tau_{\mathrm{ref}})},
\label{eq:stats}
\end{equation}
where $r_i(\tau)$ is the log return from $t_i$ to $t_i+\tau$. $S$ is denominator-free and reported in basis points. $D$ normalises by a pooled scale. $\widetilde D$ normalises per item; it looks like a natural way to make heterogeneous events comparable, and \S\ref{sec:ratio} shows it is not an estimator at all. We use $\tau_{\mathrm{ref}}=1800$\,s throughout the main text and report the legacy pipeline's $\tau_{\mathrm{ref}}=7200$\,s convention in Appendix~\ref{app:anchor}.

\section{Theory}
\label{sec:theory}

\subsection{Precision does not transfer; \texorpdfstring{$(\TPR,\FPR)$}{(TPR,FPR)} does}
Calibrating a detector means choosing a threshold, which fixes an operating point $(\TPR,\FPR)$. Those are conditional on the signal and noise distributions, not on $\pio$. Precision is then determined by \eqref{eq:bayes} in any pool where $\pio$ is known. The failure mode is not subtle once stated, but note what makes it invisible in practice: at $\pio\approx1$, precision is near 1 for \emph{any} $\FPR$, so a validation at high $\pio$ carries essentially no information about $\FPR$ --- the very quantity that governs behaviour at low $\pio$. High-prevalence validation is not weak evidence about low-prevalence deployment; it is close to no evidence.

\subsection{The mixture identity}
\label{sec:mixture-theory}
Let $E,P$ partition a detected index into true events and phantoms, with intensities $w_E,w_P$ under the estimator's own scale. Because the anchor sign is defined identically on both subsets, for every $\tau$
\begin{equation}
D^{\mathrm{det}}(\tau)\;=\;\frac{(1-c)\,w_E\,D_E(\tau)\;+\;c\,w_P\,D_P(\tau)}{(1-c)\,w_E + c\,w_P}
\;=\;\frac{(1-c)\,D_E(\tau)+c\,\rho\,D_P(\tau)}{(1-c)+c\,\rho}.
\label{eq:mixture}
\end{equation}
This is an identity, not a model: it holds for any detector, any $\tau$, any data. Its content is diagnostic. A dataset owner who knows $(c,\rho,D_E,D_P)$ knows exactly what their published curve measures. A user holding only $D^{\mathrm{det}}$ cannot invert it, which is precisely the difficulty: contamination is not visible from inside the contaminated dataset.

Two consequences. First, contamination is \emph{not} additive noise. $D_P$ has a shape, and when the detector selects on the same quantity that defines the phenomenon, that shape resembles the phenomenon. Second, the weights involve $\rho$, so a dataset with high precision but loud phantoms can be more damaged than one with low precision and quiet phantoms.

\subsection{The phantom null is a selection-conditioned null}
Phantoms are, by construction, windows accepted by the rule. Under a pure-noise null, conditioning on a large move at the acceptance horizon induces structure in any statistic anchored near that horizon --- not because the noise contains an event, but because the conditioning event is itself part of the statistic. The correct null for $D_P$ is therefore not "flat'' and not the unconditional random walk, but the \emph{detector's transfer function}: noise passed through the same acceptance rule and aggregated by the same estimator. We construct this null in \S\ref{sec:mixture}.

\subsection{A normalisation with no finite expectation}
\label{sec:ratio-theory}
\begin{proposition}
Let $r_{\mathrm{ref}}$ have a density $f$ that is continuous and strictly positive in a neighbourhood of $0$. Then $\mathbb{E}\big[|r(\tau)/r_{\mathrm{ref}}|\big]=\infty$, and the estimator $\widetilde D(\tau)=\frac1n\sum_i r_i(\tau)/r_i(\tau_{\mathrm{ref}})$ of \eqref{eq:stats} has no finite expectation.
\end{proposition}
\begin{proof}[Proof sketch]
$\mathbb{E}|1/r_{\mathrm{ref}}|=\int f(x)/|x|\,dx$ diverges logarithmically at the origin whenever $f(0)>0$. Conditional on $r(\tau)$ having non-degenerate conditional mean near $r_{\mathrm{ref}}=0$, the same divergence carries to the ratio.
\end{proof}
The sample mean of $\widetilde D$ is finite for any finite $n$, so the pathology is silent: one obtains a number, it looks stable across reruns on the same data, and it does not converge to anything as $n$ grows. It is dominated by the items with the smallest $|r_{\mathrm{ref}}|$ --- which, in an event dataset, are the items where nothing happened. \S\ref{sec:ratio} measures the size of the resulting artefact.

\section{A dataset with recoverable phantoms}
\label{sec:data}

\paragraph{Price series.} E-mini S\&P~500 futures (front-month continuous, \texttt{ES.c.0}), best-bid-offer at 1-second resolution from a CME-derived vendor feed, in windows of $t_0\pm2$\,h. Total spend \$2.04.

\paragraph{Official index.} 555 announcement windows over 2010--2026: 212 CPI and 203 non-farm payroll releases (release dates from the Federal Reserve Bank of St.\ Louis release calendar; release times from the statutory 08:30 ET convention), and 140 FOMC statements (dates and per-statement release times verified against the Federal Reserve's own press-release pages; 14:15 ET before March 2013, 14:00 ET after, with seven unscheduled statements verified individually).

\paragraph{Detected index.} A magnitude-threshold scan of the same price series: a window is accepted when the 2-minute move exceeds ten times a per-second standard deviation estimated on the pre-window. Full specification in Appendix~\ref{app:legacy}. The scan was run separately on three candidate pools, which is what makes this dataset useful:
\begin{center}
\begin{tabular}{lll}
\toprule
Pool & Definition & $|C|$ \\
\midrule
FOMC & the Federal Reserve's published meeting calendar & 140 \\
NFP & the first Friday of each month & 200 \\
CPI & business days 6--16 of each month & 1569 \\
\bottomrule
\end{tabular}
\end{center}

\paragraph{Screens.} Applied identically to every label set before any comparison: a pre-window futures-roll screen (any 1-second log return above 20\,bp strictly \emph{before} $t_0$; post-$t_0$ jumps are deliberately excluded so the screen cannot delete announcements --- see App.~\ref{app:accidents}, Accident~4) and a degraded-quote screen. 415 official windows survive for the terminal-anchor analysis and 417 for the 30-minute-anchor analysis, all from $\sim$2013 onward; the vendor's 1-second book history does not extend reliably before then.

\paragraph{Why this is a good testbed and where it is limited.} The two indices are generated by causally independent processes: one by statistical agencies and a central bank, the other by a threshold on prices. Neither was constructed with the other in view. That independence is what makes phantom identification credible. The limitations are equally clear and we state them here rather than only in \S\ref{sec:limits}: one instrument, one detector family, one class of official index.

\section{Results}
\label{sec:results}

\subsection{One detector, three pools, three datasets}
\label{sec:pools}

\begin{table}[htbp]
\centering
\caption{The same detector deployed on three candidate pools. ``Phantom rate'' is $n_P/n_{\mathrm{legacy}}$, i.e.\ one minus the precision of the detector \emph{in that pool}. Counts are raw detections before quality screens; the screened counts used in \S\ref{sec:mixture} are $n_E=308$, $n_P=473$.}
\label{tab:overlap}
\begin{tabular}{lrrrrrrr}
\toprule
Pool & $|C|$ & $\pio$ & $n_{\mathrm{official}}$ & $n_{\mathrm{legacy}}$ & $n_E$ & $n_P$ & phantom rate \\
\midrule
FOMC & 140  & 1.000 & 140 & 126 & 126 & 0   & \textbf{0.0\%} \\
NFP  & 200  & 0.820 & 203 & 122 & 111 & 11  & \textbf{9.0\%} \\
CPI  & 1569 & 0.105 & 212 & 596 & 109 & 487 & \textbf{81.7\%} \\
\midrule
All  & 1909 & --- & 555 & 844 & 346 & 498 & 59.0\% \\
\bottomrule
\end{tabular}
\end{table}

Table~\ref{tab:overlap} is the paper's central measurement. Three observations.

\emph{The spread is the finding.} One detector, one price series, one implementation, three pools --- and three qualitatively different datasets. On the Fed's calendar the detector is perfect. On business days 6--16, four out of every five accepted items never happened.

\emph{Aggregation conceals it.} The pooled phantom rate is 59\%, which reads as a mediocre-but-usable detector. Per-pool, one of the three datasets is majority-synthetic. A practitioner who reports a single precision figure for a detector deployed across heterogeneous pools has reported a number that describes none of the deployments.

\emph{Phantoms dominate in count.} On the screened sample, $n_P=473$ against $n_E=308$: $c=0.606$. A majority of the published event list consists of events that did not occur.

\subsection{Cross-pool transfer: naive $+422\%$, Bayes $-3.3\%$}
\label{sec:transfer}

The three pools give a natural calibration/deployment split. Calibrate on the two high-prevalence pools (FOMC, $\pio=1.00$; NFP, $\pio=0.82$), deploy on the low-prevalence one (CPI, $\pio=0.105$).

A practitioner who treats precision as a detector property transfers the mean of the observed precisions, $\tfrac12(1.000+0.910)=0.955$. The measured CPI precision is $0.183$: an error of $\mathbf{+422\%}$. The Bayes expression \eqref{eq:bayes}, using each pool's own $\pio$, that pool's $\TPR$, and the detector's analytic $\FPR=0.361$ (derived in App.~\ref{app:accidents}, Accident~1, and confirmed by Monte Carlo to within 2\%), predicts $0.177$ against $0.183$ measured: an error of $\mathbf{-3.3\%}$. The corresponding errors are $-1.6\%$ for NFP and $0\%$ for FOMC. Figure~\ref{fig:precision} shows all three.

\begin{figure}[htbp]
\centering
\includegraphics[width=0.78\textwidth]{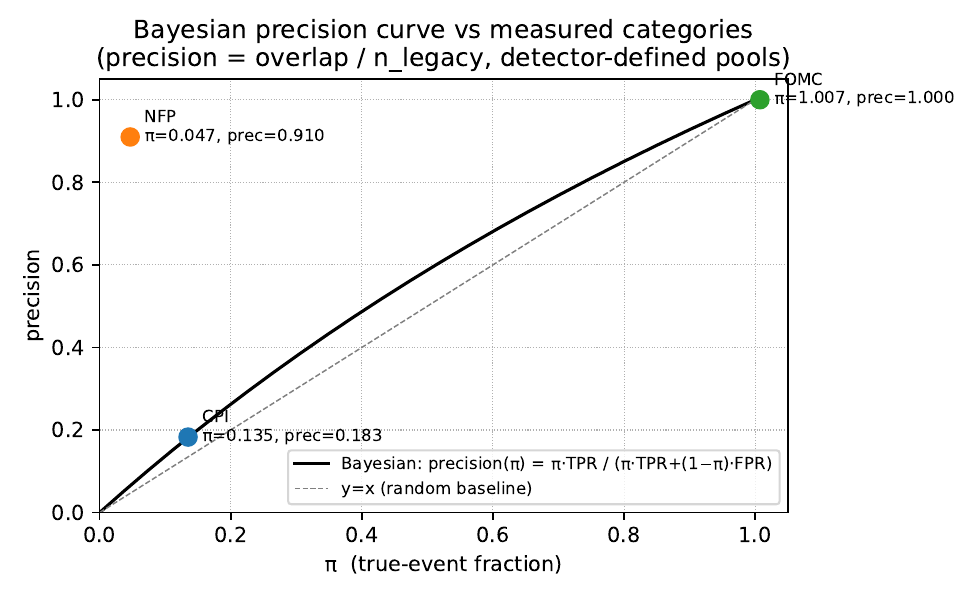}
\caption{Precision as a function of pool prevalence. Solid and dashed black: \eqref{eq:bayes} at the two measured $\TPR$ values, with $\FPR=0.361$ fixed. Dotted grey: the naive belief that precision is a property of the detector. Three measured deployments of one detector fall on the curve. The red marker is the naive transfer from the two high-$\pio$ pools to the low-$\pio$ pool; the dashed arrow is the resulting error.}
\label{fig:precision}
\end{figure}

We stress what is and is not being claimed. \eqref{eq:bayes} is elementary and we claim no novelty for it. The claim is that its consequences are large, are realised in a real dataset built by competent people, and are not detectable from inside that dataset. The naive transfer is not a straw man: it is what "we validated the detector and got 95\% precision'' means when the validation pool and the deployment pool differ in prevalence by an order of magnitude.

\subsection{Mixture decomposition: the contaminated curve, exactly}
\label{sec:mixture}

Applying \eqref{eq:mixture} to the screened sample:
\[
n_E=308,\quad n_P=473,\quad c=0.606,\quad w_E=6.38\,\mathrm{bp},\quad w_P=3.76\,\mathrm{bp},\quad \rho=0.590,
\]
and the identity check
\[
\max_\tau\big|D^{\mathrm{det}}_{\mathrm{pred}}(\tau)-D^{\mathrm{det}}_{\mathrm{obs}}(\tau)\big| \;=\; 1.1\times10^{-16},
\]
i.e.\ exact to machine precision (Figure~\ref{fig:mixture}, right). This is first a verification that the bookkeeping is sound, and second a statement with content: the published curve of the detected dataset contains exactly the information in \eqref{eq:mixture} and nothing else. There is no residual to interpret.

\begin{figure}[htbp]
\centering
\includegraphics[width=\textwidth]{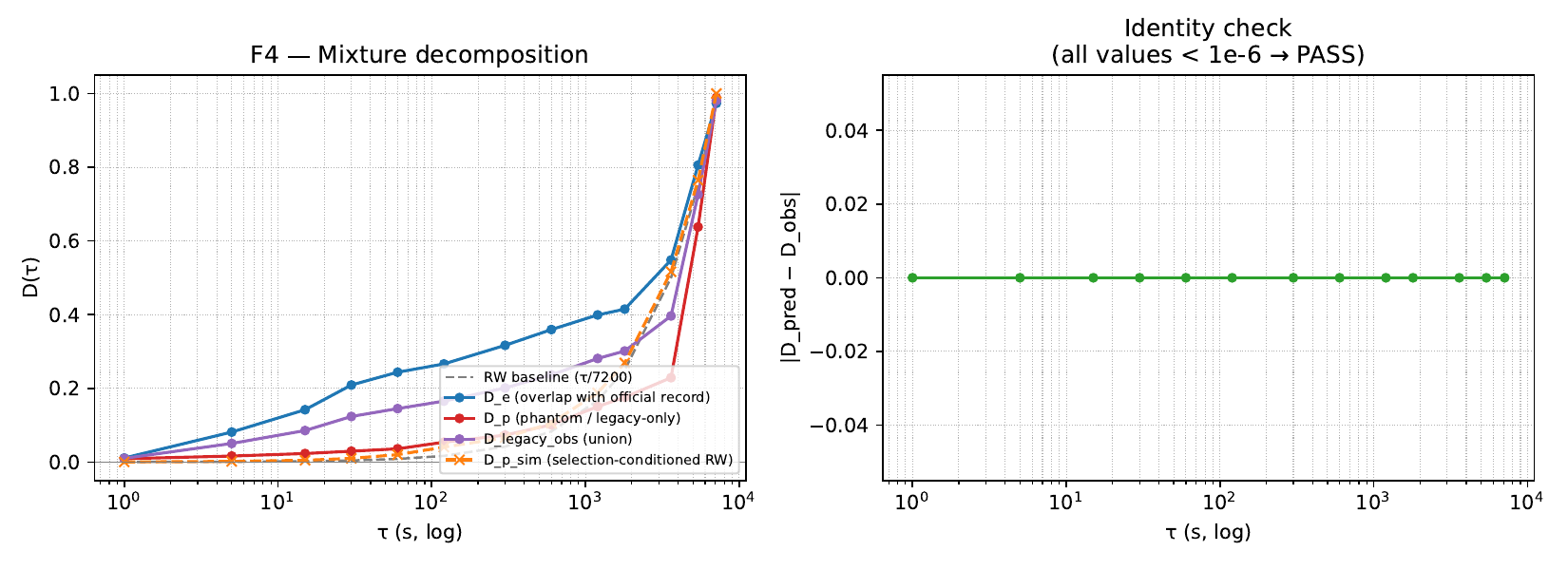}
\caption{\textbf{Left:} the detected dataset's curve decomposed into its true-event component $D_E$, its phantom component $D_P$, and the exact mixture, with the unconditional random-walk baseline and the selection-conditioned null for $D_P$. \textbf{Right:} identity \eqref{eq:mixture} verified to machine precision.}
\label{fig:mixture}
\end{figure}

\paragraph{What the phantom component is.} $D_P$ is not flat and not the unconditional random walk. It rises with $\tau$, qualitatively like $D_E$. Compared against the selection-conditioned null of \S\ref{sec:theory} --- pure noise passed through the detector's exact acceptance rule and the same estimator --- it is reproduced in shape but not in level (RMSE 0.092 over $\tau\le600$\,s, against $D_P$ values in the range 0.02--0.07; the null sits systematically low). We therefore state the weaker claim that the data support: the phantom component's \emph{shape} is generated by the acceptance rule rather than by event content, while its level is not fully accounted for by our null. We regard the residual as an open question rather than evidence of event content in the phantoms.

\paragraph{Why $\rho$ matters.} Here $\rho=0.59$: phantoms are quieter than true events, so their mixture weight is below their count share. A dataset with the same $c$ but $\rho>1$ would be substantially more damaged. Reporting precision alone ($=1-c$) does not distinguish these cases.

\subsection{The direction of contamination is a property of the estimator}
\label{sec:direction}

This is the finding we expect to generalise furthest beyond the domain.

Under the terminally normalised statistic used by the original pipeline, contamination \emph{inflates}: at $\tau=5$\,s, legacy $0.082$ versus official $0.068$; at $300$\,s, $0.317$ versus $0.271$. The contaminated dataset looks like it has \emph{more} structure than the truth. Under the denominator-free statistic $S(\tau)$ on the same windows, the ordering reverses: official $0.79$\,bp $>$ legacy $0.48$ $>$ phantom-only $0.18$ at $\tau=5$\,s, and $2.53>1.65>0.83$ at $300$\,s. Contamination \emph{dilutes}. Figure~\ref{fig:curves} shows both.

\begin{figure}[htbp]
\centering
\includegraphics[width=\textwidth]{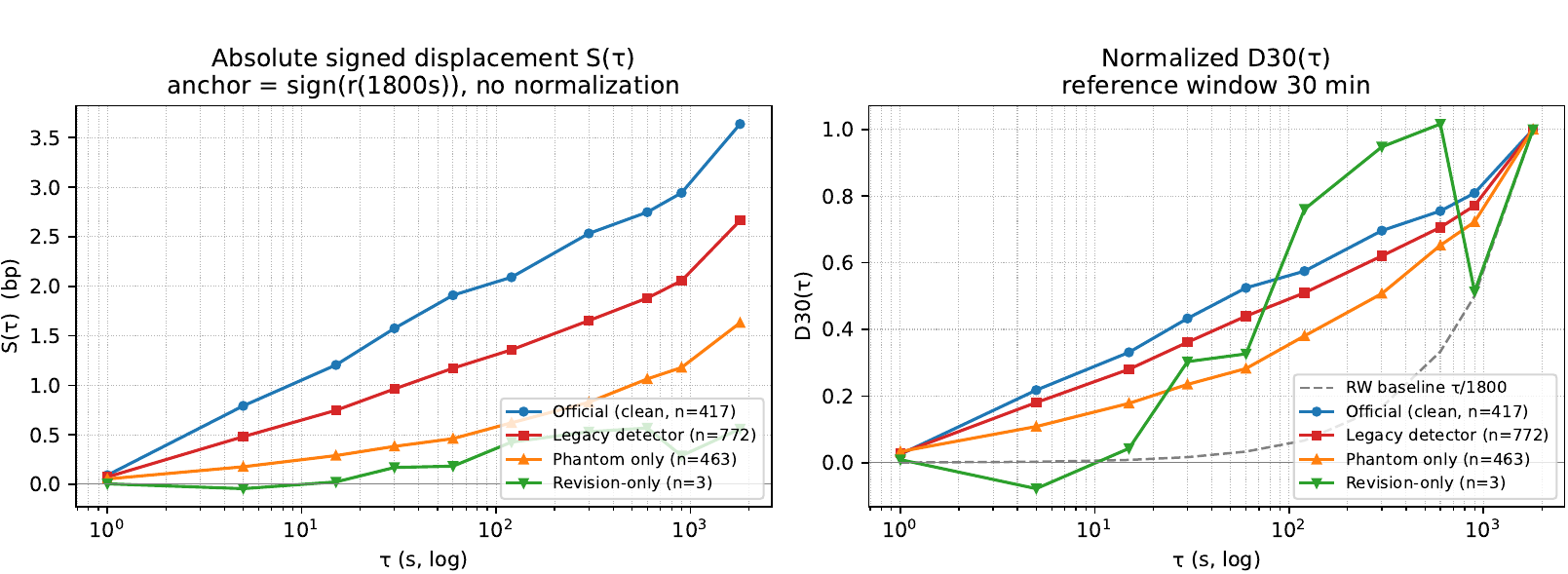}
\caption{Same price windows, two label sets, two estimators. \textbf{Left:} denominator-free signed displacement $S(\tau)$ in bp. \textbf{Right:} normalised $D_{30}(\tau)$ with its random-walk baseline $\tau/1800$. Official $n{=}417$, detected $n{=}772$, phantom-only $n{=}463$. \emph{The revision-only control is plotted at $n{=}3$ and is not interpretable}; only 3 of the 17 calendar-identified control windows have price coverage (App.~\ref{app:revision}). Under terminal normalisation (App.~\ref{app:anchor}) the official/detected ordering is reversed relative to the left panel.}
\label{fig:curves}
\end{figure}

The mechanism is \eqref{eq:mixture} together with the denominator. $S$ has no denominator, so the mixture simply pulls the true curve toward the phantom curve: dilution, as intuition expects. $D^{\mathrm{term}}$ divides by a scale estimated on the same pooled windows. Phantoms enter both numerator and denominator, and because $\rho<1$ they depress the denominator proportionally more than the numerator at short horizons, inflating the ratio. The estimator's own scale is contaminated.

The consequence for practice is sharper than the mechanism. A user holding only a contaminated dataset cannot determine which regime they are in, because the diagnosis requires $\rho$, which requires knowing which items are phantoms. Statements of the form "contamination will attenuate my effect, so my estimate is conservative'' --- a very common defence in applied work --- are not available. Attenuation is one of two possible directions and which one occurs depends on a quantity the analyst cannot measure.

\subsection{A normalisation that returns 0.40 where the estimator returns 0.10}
\label{sec:ratio}

The pathology of \S\ref{sec:ratio-theory} is not hypothetical; we found it in our own pipeline, and it accounts for a factor-of-four discrepancy that took two rounds to resolve. On an identical set of 335 events, the two aggregations of \eqref{eq:stats} give
\[
\widetilde D(5\mathrm{s}) = 0.3997 \qquad\text{versus}\qquad D(5\mathrm{s}) = 0.1020 .
\]
Splitting by the magnitude of the reference return makes the source visible: on quiet windows ($|r_{\mathrm{ref}}|$ below the sample median) $\widetilde D(5\mathrm{s})=0.415$; on active windows, $0.097$. The quiet windows, where by construction nothing happened and the post-event path is uncorrelated with the anchor sign, dominate the average because they carry the largest $1/|r_{\mathrm{ref}}|$ weights.

Per-item normalisation is an appealing move --- it is how one makes heterogeneous events comparable --- and its failure is silent, because the estimator returns a plausible number that is stable across reruns. We report it because we suspect it is not rare: any statistic of the form "mean of per-item normalised response'' with a signed, mean-zero denominator has this property.

\subsection{Isolating the statistical channel}
\label{sec:sim}

Finally we separate the mechanism from everything domain-specific. $M=5000$ candidate windows per configuration, 10 seeds; a fraction $\pio$ carry a delayed smooth jump ($\lambda\in\{2,10,60\}$\,s, amplitude $U(3,10)\times\sigma_{2\mathrm{m}}$), the rest are pure random walk. Two detectors at \emph{matched nominal strictness}: detector~A thresholds $|\Delta p_{2\mathrm{m}}|/\sigma_{1\mathrm{s}}$ at 10; detector~B thresholds $|\Delta p_{2\mathrm{m}}|/(\sigma_{1\mathrm{s}}\sqrt{120})$ at 3. Both read as strict; their actual false-positive rates differ by two orders of magnitude (0.361 versus 0.0027). A self-check at $\pio=1$ confirms the harness recovers the injected response (relative bias 0.039).

\begin{figure}[htbp]
\centering
\includegraphics[width=\textwidth]{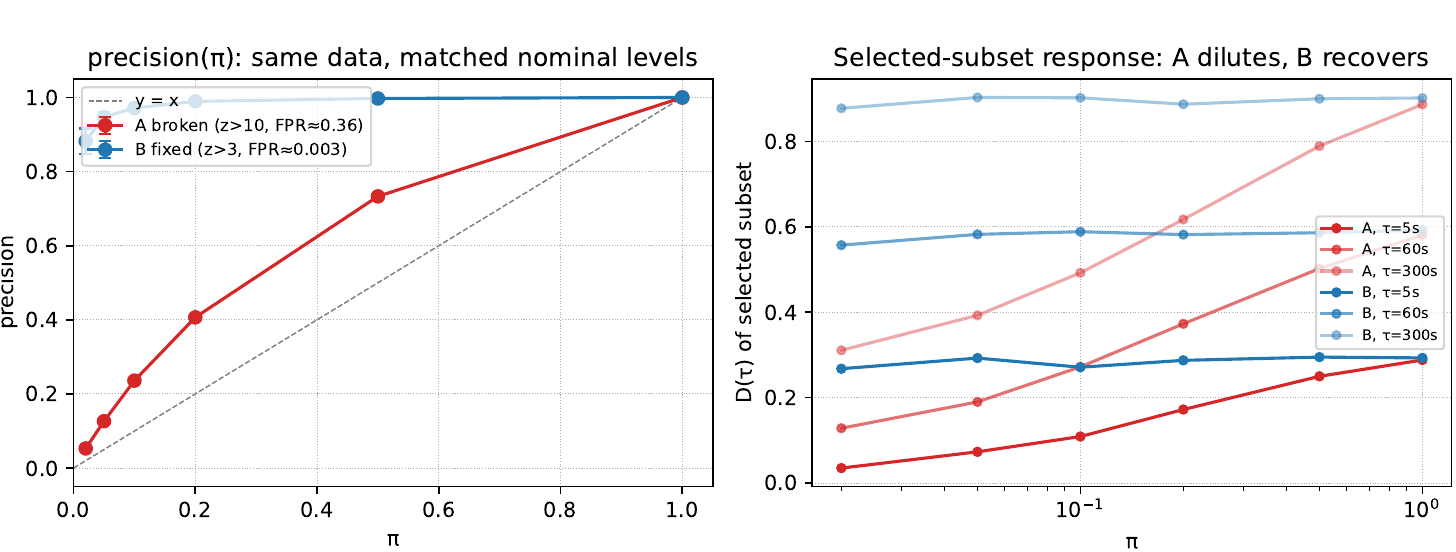}
\caption{Two detectors at matched \emph{nominal} strictness on identical data. \textbf{Left:} at $\pio=0.02$, A attains precision 0.053 while B attains 0.882. \textbf{Right:} the response curve measured on each detector's selected subset, as a function of $\pio$: A's estimate drifts with the composition of its own selection; B's is stable.}
\label{fig:simulator}
\end{figure}

At $\pio=0.02$ the two equally strict detectors yield precisions of 0.053 and 0.882 on the same data. Nothing about the underlying events changed; only a time-scale normalisation inside the acceptance statistic did. This is the minimal mechanism behind Table~\ref{tab:overlap}: a detector-defined dataset is the detector's normalisation choices stamped onto data, and those choices are not recoverable from the output.

\subsection{Scope, and a pre-registered claim we did not get to make}
\label{sec:scope}

Before running the analysis we registered a threshold for an absolute claim: if the clean official index gave $D(5\mathrm{s})\ge0.30$ we would claim a fast price-discovery effect, and at $\le0.15$ we would not. The measured value on 415 clean official windows is $\mathbf{0.068}$. We therefore make no such claim. We report this because the pre-registration is what makes the rest of the paper's numbers credible, and because the temptation to relax the threshold after seeing 0.068 was real and is exactly the behaviour the paper argues against.

\section{Related work}
\label{sec:related}

\paragraph{Base rates and precision under imbalance.} That posterior probability depends on prevalence is the base-rate fallacy \citep{barhillel1980}; its consequences for classifier evaluation are treated by \citet{hand2009}, and \citet{saito2015} show specifically that precision-based summaries are prevalence-dependent in a way that ROC summaries are not. Our contribution is not the mechanism but an end-to-end measurement of it in a deployed dataset with recoverable ground truth, plus the mixture and estimator-dependence results that follow.

\paragraph{Weak, distant, and self-supervision.} Distant supervision \citep{mintz2009} and programmatic weak supervision \citep{ratner2017} construct labels by rule application; the label-model literature estimates and corrects labelling-function accuracies, but typically assumes the accuracies are pool-stable, which is precisely the assumption Table~\ref{tab:overlap} violates. Pseudo-labelling suffers a related confirmation bias \citep{arazo2020}. Positive-unlabelled learning \citep{elkan2008} formalises learning when negatives are unobserved and makes prevalence explicit; our setting is its evaluation-side dual.

\paragraph{Label noise.} A large literature studies learning under label noise \citep{natarajan2013,frenay2014}. That work almost always models noise as a stochastic corruption applied to correct labels. The contamination here is not of that form: phantoms are generated by the same rule that generates true positives, so the corruption is systematically correlated with the features and carries its own signal shape (\S\ref{sec:mixture}). Noise-robust methods that assume class-conditional random flips do not address it.

\paragraph{Dataset quality and documentation.} \citet{northcutt2021} document pervasive label errors in benchmark test sets, and \citet{gebru2021} propose documentation standards. Our protocol (\S\ref{sec:protocol}) is in this tradition and specialises it to the detector-defined case, where the relevant fields are $(\pio,\TPR,\FPR,c,\rho)$ rather than provenance narrative.

\paragraph{Event studies.} The applied setting is the event study \citep{mackinlay1997}; the specific phenomenon our detector was built to find is announcement response \citep{andersen2003}. We use this literature only as the source of a realistic detector and a realistic downstream estimand.

\section{A reporting protocol for detector-defined datasets}
\label{sec:protocol}

\begin{enumerate}\itemsep1pt
\item \textbf{Observed indices first.} Use a published index wherever one exists; detection is a last resort, not a convenience.
\item \textbf{Validate in the deployment pool.} Precision must be measured, or derived via \eqref{eq:bayes}, on the pool where the detector is actually run. Cross-prevalence transfer of precision is not permitted. High-prevalence validation is near-zero evidence about $\FPR$.
\item \textbf{Report five quantities.} $\pio$, $\TPR$, $\FPR$, $c$, $\rho$. Precision alone conceals which one moved, and $\rho$ determines whether a given $c$ is tolerable.
\item \textbf{Declare the time scale of every $z$-like statistic.} Any ratio of a scale-$t$ quantity to a per-unit scale must be $\sqrt{t}$-normalised explicitly, and the reference window of every normalisation stated.
\item \textbf{Ship an analytic null with every estimator.} A closed-form null for the acceptance rule is what turns "our detector fires often'' into a measured $\FPR$. When the estimator is applied to a selected subset, the null must be the \emph{selection-conditioned} null, not the unconditional one.
\item \textbf{Avoid per-item normalisation by a signed, mean-zero denominator.} Such estimators have no finite expectation (\S\ref{sec:ratio-theory}); use a pooled denominator or none.
\item \textbf{Run an oracle power check before comparing methods.} Compute the bound achieved by a method given perfect ground truth. If it is within seed noise of the baseline, the configuration cannot distinguish methods, and this is a property of the design rather than of the methods (App.~\ref{app:oracle}).
\end{enumerate}

\section{Limitations}
\label{sec:limits}

\begin{enumerate}\itemsep1pt
\item \textbf{One detector family.} All measurements concern magnitude-threshold detectors. Equations \eqref{eq:bayes} and \eqref{eq:mixture} hold for any detector; our numbers do not transfer to change-point, likelihood-ratio, or learned detectors.
\item \textbf{One instrument, one asset class, one index family.} Prevalences and phantom rates are pool properties. The mechanism is general; the magnitudes are not.
\item \textbf{Sample coverage.} 415 clean official windows, effectively 2013 onward; 68 pre-2013 windows returned no data at the vendor's 1-second resolution. Cross-decade comparability is limited.
\item \textbf{The selection-conditioned null is incomplete.} It reproduces the shape of $D_P$ but not its level (\S\ref{sec:mixture}).
\item \textbf{Relative claims only.} Having abandoned the terminally normalised statistic for headline use, we support relative conclusions between datasets and estimators, not absolute statements about response speed.
\end{enumerate}

\section*{Reproducibility statement}
All event indices, phantom labels, screens, curves, and simulator code are released, together with a per-number traceability map, the full decision log for every pre-registered branch point, and an accident log (App.~\ref{app:accidents}). The price data cannot be redistributed under the vendor's licence; we release the exact request manifests (symbol, schema, window, checksum) so that an identical dataset can be reconstructed.

\appendix

\section{Six accidents in this paper's own pipeline}
\label{app:accidents}

Each accident below occurred during construction of this paper's dataset, was resolved by a pre-registered action, and is documented in full in the released log. We present them as worked instances of the paper's mechanism. Each entry gives the error, how it was found, the blast radius, and the fix.

\paragraph{Accident 1: the unscaled $z$-score.} The detector's statistic divides a 2-minute move by a \emph{per-second} standard deviation, omitting the $\sqrt{120}$ scaling. Under a random-walk null this gives $\FPR = 2\big(1-\Phi(10/\sqrt{120})\big)=0.361$: the threshold fires on roughly one in three ordinary windows. Consequence: 487 phantom CPI detections. Table~\ref{tab:overlap}'s 81.7\% phantom rate \emph{is this bug, measured}. Found by deriving the closed-form null and checking it against simulation (Rule~5), neither of which the original pipeline had. The bug is retained in \S\ref{sec:sim} as detector~A because it is realistic.

\paragraph{Accident 2: the hallucinated metadata correction.} An automated pipeline step "corrected'' the FOMC statement release time from 14:00\,ET to 13:30\,ET for all statements after a cutoff, citing a secondary report of a central-bank announcement that, on inspection of the institution's own press-release pages, does not exist. Both primary sources checked state release at 14:00\,ET. Blast radius: 18 ground-truth timestamps shifted by 30 minutes, and one full response-curve computation silently rerun on the corrupted index before the error was caught. Fix: rollback; a hard rule that no metadata may be modified without a primary-source URL recorded in the manifest; per-statement verification against primary pages. This is the paper's thesis in miniature --- a confident pipeline, an unmeasured false-positive rate on the task of "correcting metadata'', and a ground truth that then validated itself.

\paragraph{Accident 3: positional indexing on a series with gaps.} The first response-curve implementation located $t_0$ by row offset, assuming one row per second. The vendor emits a row only on quote update, so the offset drifted by up to 48 minutes in sparse periods and the signal averaged away. Fix: timestamp-based alignment throughout.

\paragraph{Accident 4: a cleaning filter that deleted the signal.} A futures-roll filter flagged any 1-second return above 20\,bp anywhere in the window. Of 86 official windows so flagged, 68 had their largest jump within $\pm5$\,s of $t_0$: the filter was removing announcement reactions as artefacts. Fix: the roll screen is evaluated strictly on the pre-window, which cannot contain the reaction. \emph{Lesson, and a general one:} any cleaning rule whose support intersects the event window can become selective deletion of the target signal --- structurally identical to the contamination this paper studies, with the sign flipped.

\paragraph{Accident 5: the same indexing bug, in a different file.} After Accident~3 was fixed in the estimator, the identical row-offset pattern reappeared in the screening code's pre-window slice. In high-activity sessions the "pre-window'' slice crossed $t_0$ and the largest announcement reactions were deleted as rolls --- for instance the November~2022 CPI release, whose true pre-window maximum was 3.3\,bp and whose post-event move was 126\,bp. The months with the strongest signal were deleted at the highest rate. Fix: timestamp masks, a repository-wide audit for row-versus-time indexing, and a regression test. Clean official windows went from 343 to 415.

\paragraph{Accident 6: a factor-of-four estimator artefact.} Two implementations of the response curve disagreed by $4\times$ on identical events. The cause was the per-item normalisation of \S\ref{sec:ratio}. Resolution is reported in the main text rather than only here because we consider the estimator pathology a finding rather than a mishap.

\section{Oracle power checks and a retracted experiment}
\label{app:oracle}
For any timing-sensitive evaluation, the \emph{oracle} is the same downstream method given the true timestamps. The oracle-minus-baseline gap upper-bounds the total effect any alignment or denoising method can produce in that configuration. An earlier version of this project evaluated an alignment module in a configuration whose oracle gap was 0.024 CRPS against a comparable seed-noise floor: no method comparison there could have detected anything, regardless of method quality. The experiment was retracted before producing conclusions. Rule~7 operationalises the check as a pre-registration step. The retracted results are released for completeness and should not be cited.

\section{The revision-only negative control}
\label{app:revision}
Same-month double releases --- where a scheduled release carries only a revision to previously published figures and no new reference-period data --- provide events with the exact metadata structure of real releases but, by construction, no new information. A calendar rule identifies 17 such windows (13 CPI, 4 payrolls). Only 3 have usable price coverage in our sample. We report this because the control was pre-registered, and we do \emph{not} interpret it: $n=3$ is insufficient, and the curve shown in Figure~\ref{fig:curves} is included only for completeness. Constructing this control properly requires price coverage we could not obtain within the project's data budget, and we flag it as the single cheapest improvement available to a follow-up.

\section{The terminally normalised statistic}
\label{app:anchor}
For completeness we give the legacy 2-hour-anchor convention on the screened sample. At $\tau=5$\,s: official 0.068, detected 0.082; at $300$\,s: 0.271 and 0.317; random-walk baseline $\tau/7200$. This is the statistic underlying the original dataset's published curve. Its ordering (detected above official at every horizon) reverses under $S(\tau)$, which is the demonstration of \S\ref{sec:direction}. Per-category values at $\tau\in\{5,30,300,3600\}$\,s are released in \texttt{f2\_dtau\_by\_subcat.csv}.

\section{Detector specification and screens}
\label{app:legacy}
The detector scans the front-month continuous ES series at 1-second best-bid-offer resolution and accepts a window when $|\Delta p_{2\mathrm{m}}|$ exceeds ten per-second standard deviations estimated on the 5 minutes preceding $t_0$ (Accident~1's unscaled form). Windows are $t_0\pm2$\,h on bid-ask mid. Screens: (i) pre-window roll screen, any 1-second log return above 20\,bp strictly before $t_0-1$\,s, evaluated on timestamp masks; (ii) degraded-quote screen, windows with more than 5\% missing bid or ask. Screens are applied identically to every label set before any comparison.

\end{document}